\documentclass{article}
\usepackage{amsmath,amssymb,amsthm}
\usepackage{graphicx}
\usepackage{natbib}
\newtheorem{theorem}{Theorem}

\usepackage[a4paper, total={6in, 8in}]{geometry}
\usepackage{hyperref}

\begin{document}

\title{Improving Network Threat Detection by Knowledge Graph, Large Language Model, and Imbalanced Learning}
\author{
Lili Zhang\thanks{Hewlett Packard Enterprise}, 
Quanyan Zhu\thanks{New York University}, 
Herman Ray\thanks{Kennesaw State University}, 
Ying Xie\thanks{Kennesaw State University}
}
\date{}
\maketitle

\abstract{%
Network threat detection is challenging due to the complex nature of attack activities and the limited availability of historical threat data to learn from. To help enhance the existing methods (e.g., analytics, machine learning, and artificial intelligence) to detect the network threats, we propose a multi-agent AI solution for agile threat detection. In this solution, a Knowledge Graph is used to analyze changes in user activity patterns and calculate the risk of unknown threats. Then, an Imbalanced Learning Model is used to prune and weigh the Knowledge Graph, and also calculate the risk of known threats. Finally, a Large Language Model (LLM) is used to retrieve and interpret the risk of user activities from the Knowledge Graph and the Imbalanced Learning Model. The preliminary results show that the solution improves the threat capture rate by 3\%-4\% and adds natural language interpretations of the risk predictions based on user activities. Furthermore, a demo application has been built to show how the proposed solution framework can be deployed and used.
}%







\section{Introduction}\label{sec:Intro}

Network threats have brought significant financial losses and public safety issues in recent years. The total reported loss from cybercrimes is more than \$12. 5 billion in the US in 2023 according to the FBI's Internet Crime Complaint Center (IC3) report ~\citep{IC3}. Moreover, public safety systems face increasing disruption in emergency communication systems and operations due to malicious attacks ~\citep{CISA}. These are caused by more complicated and new network attack activities that are not detected in time ~\citep{zhu2012guidex}. This presents a significant need for Agile Threat Detection, which aims to identify and respond to evolving threats rapidly and proactively ~\citep{zhu2024foundations}. 

The analytics, machine learning (ML) and artificial intelligence (AI) methods have been widely used by researchers and practitioners to discover the patterns of known threats and detect unusual signals of unknown threats from the activities of users. Traditional ML/AI models typically need a lot of historical data to learn from to guarantee good model performance. However, there are very limited historical data on known threats that have been observed but are not detected every time they occur. And there is no data on unknown threats that have never been observed before. These challenge traditional ML/AI models to predict network threats accurately. 

Compared to other ML / AI models only, Knowledge Graph shows a higher efficiency in analyzing user activities and their relationships to discover abnormalities. However, it has three challenges. The first is to prune and weigh the information properly in the graph to filter out weak or redundant information for network threats. The second is to include large texts as a part of graphs and graph analysis. The third is to unravel, diagnose, and interpret the complex activities and relationships of users in the graph. 

To overcome the challenges above, we propose to better detect the network threats by the combination of Knowledge Graph, Imbalanced Learning, and Large Language model (LLM) through a multi-agent AI framework. In this framework, the Knowledge Graph is used to analyze changes in the user activity pattern and calculate the risk of unknown threats, the Imbalanced Learning Model is used to prune and weigh the Knowledge Graph and also calculate the risk of known threats, and LLM is used to retrieve and interpret the user activities from the Knowledge Graph and the Imbalanced Learning Model.

Using LLMs empowered with Knowledge Graphs enhanced by Imbalanced Learning as a set of AI agents, an adaptive and real-time monitoring framework can be implemented to achieve fast and early detection of malicious behaviors. This approach integrates the strengths of LLMs in contextual reasoning with the structured relationship modeling capabilities of Knowledge Graphs to monitor, predict, and explain potential threats as they unfold. The synergy between these components ensures both depth and immediacy in threat detection, making the system highly effective in dynamic environments.

\section{Related Work}\label{sec:RelatedWork}
\subsection{Knowledge Graphs}
A \emph{knowledge graph} is a data structure that encodes entities as nodes and relations as edges, often with rich attributes on both. By capturing complex multi-relational data, knowledge graphs facilitate reasoning about connections ~\citep{Hogan2021}. In addition, knowledge graphs are extremely efficient in representing sparse big data based on their relationships and discovering abnormal patterns ~\citep{ma2021comprehensive} ~\citep{janev2020knowledge}.

Thanks to the flexibility of running different algorithms on graph data structures (e.g., similarity, centrality, community detection, path finding, shortest path, link prediction), knowledge graphs have become popular in many domains ~\citep{zhou2020survey} ~\citep{huang2019constructing} ~\citep{zhao2021anomaly} ~\citep{zhang2021measuring}, including cybersecurity, for integrating heterogeneous logs and threat intelligence ~\citep{ren2022cskg4apt} ~\citep{chen2022apt} ~\citep{sui2023logkg} ~\citep{sikos2023cybersecurity} ~\citep{rastogi2021predicting}. In a security context, a knowledge graph might include the nodes for entities (e.g., users, hosts, processes, and files) and edges for actions (e.g., logins, file accesses), enabling graph algorithms to detect suspicious patterns.


Despite all these advantages of the Knowledge Graph, it can be challenged by weak or redundant information in the graph. And it is not an efficient practice to include large text data either on nodes or edges in a graph, for example, actual article text that a user reads. In addition, it requires a lot of expert knowledge and experience to diagnose and interpret the information in the graph.

\begin{itemize}
  \item First, how to prune and weigh the nodes and edges in a large graph properly?
  \item Second, how to incorporate large texts into the graph analysis? 
  \item Third, how to retrieve and interpret information from both a graph and an imbalanced learning model efficiently?
\end{itemize}

The above challenges can be overcome by combining the knowledge graph with the imbalanced learning and the large language model.

\subsection{Imbalanced Learning}
In cybersecurity, malicious events are typically rare compared to benign events, referred to as imbalanced data, where a data set has fewer observations in the minority class (e.g., malicious event, threat) compared with majority classes (e.g., benign event, non-threat) ~\citep{chen2018machine}. \emph{Imbalanced Learning} is the process of learning patterns from imbalanced data. 

Imbalanced learning aims to eliminate the bias of traditional ML / AI models on imbalanced data in the learning process. Traditional models maximize overall accuracy, while imbalanced learning pays more attention to the accuracy of the minority class and maintains overall accuracy at a reasonable range. To adapt ML / AI models to the imbalanced data, common imbalanced learning techniques include sampling, weighting, and thresholding ~\citep{he2013imbalanced}. In extremely imbalanced data (for example, the data set with the minority class ratio less than 1\%), the weighting approach, which is also called \emph{cost-sensitive} or \emph{weighted} classification, generally works better, especially combined with thresholding approaches ~\citep{zhang2022improving}. The main reason is that the observations in the minority class are too few to do the sampling in a representative manner. In a weighting approach, the observations in the minority class are given higher weights in the loss function. For example, in logistic regression, one can use a weighted log-likelihood loss that assigns a higher weight of greater than 1 to threat observations and a weight of 1 to benign observations ~\citep{zhang2020descriptive}.


In our problem, imbalanced learning techniques are used to prune and weigh the nodes and edges in the graph, based on their relationships to the network threat. Generally speaking, a graph is pruned and weighed on the basis of the importance of the information represented on the nodes and edges in the graph. This typically depends on the problems that are being solved and the algorithms that are used. Wu used a graph hierarchy inference method based on the Agony model to eliminate noisy nodes or edges in the graph ~\citep{wu2023task}. Chong derived the graph weights based on the graph adjacency structure ~\citep{chong2020graph}. Jarnac used bootstrapping via zero-shot analogical pruning to select relevant nodes or edges in the graph ~\citep{jarnac2023relevant}. Given that our objective is to prune and weigh the nodes and edges in the graph based on their relationships to the network threat, a supervised learning approach is more suitable for our problem. And because there are less than 1\% threat observations and more than 99\% non-threat observations in the historical data, the Imbalanced Learning techniques are specifically used in our solution. 

\subsection{Large Language Models (LLMs)}
Large Language Models are neural networks trained on massive text corpora that can understand and generate natural language ~\citep{min2023recent} \citep{Brown2020,OpenAI2023}.  Existing LLMs include ChatGPT, LLaMA, Gemini, Claude, etc. ~\citep{zhao2023survey}. LLM have been used in applications of summarizing (e.g., summarizing user reviews for brevity), inferring (e.g., sentiment classification, topic extraction), transforming text (e.g., translation, spelling, and grammar correction), expanding (e.g., automatically writing emails), and Retrieval-Augmented Generation (e.g., reference on knowledge base beyond its training data before response) ~\citep{guastalla2023application}. However, the standard LLM retrieval process is based on similarity ~\citep{steck2024cosine}, where the similarity between the user's question and the documents in the database is measured and the most similar documents are selected to answer the user’s question, as shown in Figure \ref{fig: LLM QA}. 

\begin{figure}[h!]
\centering
\includegraphics{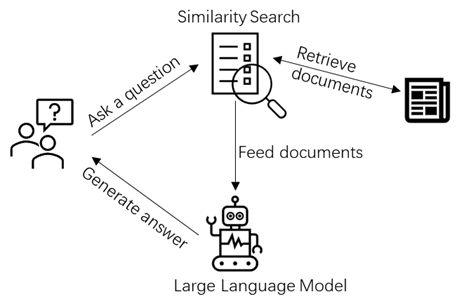}
\caption{LLM Question-Answer Process}
\label{fig: LLM QA}
\end{figure}

Recent work has explored multi-agent LLM systems where multiple models collaborate to solve tasks. These agents can pose natural-language queries to each other and to external data sources ~\citep{park2023generative} ~\citep{ni2024mechagents}. This enables dynamically intelligent interactions and collaborations among LLMs and other models and tools that are typically required to work together to solve problems ~\citep{ni2024mechagents} ~\citep{talebirad2023multi}. For example, to solve mechanical problems, a multi-agent AI platform MechAgents was developed with a comprehensive intelligent capability to retrieve and integrate relevant knowledge, theory and data, construct and execute codes, and analyze results using multiple numerical methods ~\citep{ni2024mechagents}. Another example is that multi-agent AI systems are used to enhance the decision support for smart city management, combining LLMs with existing urban information systems to process complex routing queries and generate contextually relevant responses, achieving 94-99\% accuracy ~\citep{kalyuzhnaya2025llm}.

In our framework, LLM agents serve as query-and-reasoning engines: they translate user questions into graph queries, refine queries iteratively, and interpret results in human-readable explanations.  For instance, an agent might summarize a subgraph or explain the rationale behind a flagged anomaly.

\subsection{Graph Anomaly Detection}
Detecting anomalies in graphs that evolve over time is a well-studied problem ~\citep{Akoglu2015}.  We focus on measures such as the \emph{weighted Jaccard similarity} between successive graph snapshots. Given two weighted graphs $G$ and $H$ in the same node set, we define
\[
J(G,H) \;=\; \frac{\sum_{e} \min(w_G(e),w_H(e))}{\sum_{e} \max(w_G(e),w_H(e))},
\]
which ranges in $[0,1]$.  $J=1$ if the graphs are identical and smaller values indicate a structural change.  If at time $t$ we have graph $G_t$ and at $t+1$ we add some edges or weights to get $G_{t+1}$, then $J(G_t,G_{t+1})$ quantifies how much the graph changed.  Intuitively, if few edges change, $J$ stays close to 1, but a surge of new edges (an anomaly) will drop $J$.  In Section \ref{graph_similarity_bound}, we show how the bounded updates lead to a provable bound on the change of $J$.

In summary, our work uniquely integrates a graph backbone, imbalanced learning, and cooperative LLM reasoning into one framework, despite that their individual metrics have been studied in many literature and applications separately. Knowledge-graph-based methods have been applied to security analytics and threat intelligence.  Graph-based semi-supervised learning and bootstrap methods have been used to handle noisy security data.  Imbalanced learning approaches have been adopted for rare-event detection and for log anomaly detection.  In parallel, transformer-based models have been explored in cybersecurity.  Recent work on multi-agent LLMs highlights that cooperative LLM systems can solve complex tasks through natural-language dialogue, but their use in security has been limited. 

\section{Data}\label{sec:data}

The evaluation in this paper is done on the CERT Insider Threat Test Dataset ~\citep{dataset}.  This public dataset simulates enterprise user activities (logins, file accesses, emails, etc.) for a set of users and devices, with labeled insider threats.  The CERT dataset provides detailed event logs and ground-truth threat labels, making it suitable for testing detection frameworks.  All experiments in this paper use this dataset.

\section{Proposed Methodology}

Our multi-agent framework consists of three cooperating LLM agents: two \emph{collaborators} and one \emph{supervisor}, as shown in Figure \ref{fig: AAI}.  Collaborator~1 maintains a dynamic Knowledge Graph.  At each time step $t$ it adds new events $\Delta E_t$ as edges to form $G_{t+1} = G_t \oplus \Delta E_t$, ensuring efficient online updates.  Collaborator~2 trains and applies the weighted classifier on features derived from graph entities (e.g., node degrees, subgraph patterns) to estimate threat likelihoods.  The two collaborators interact: if Collaborator~2 assigns high risk to certain events or nodes, Collaborator~1 increases their edge weights or marks them as unusual.  Conversely, edges with low risk may be pruned to focus the graph on likely threats.

\begin{figure}[h!]
\centering
\includegraphics[scale=0.37]{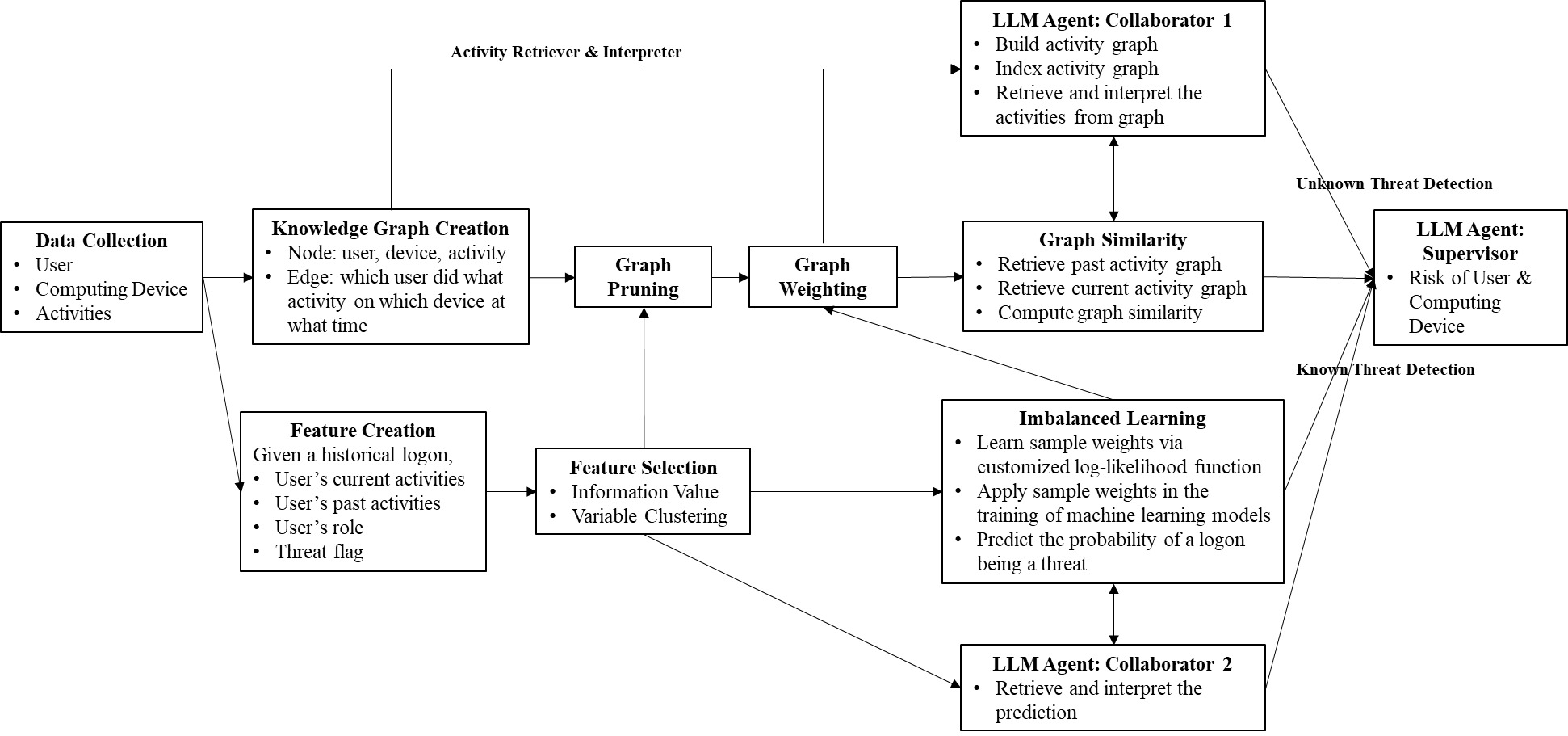}
\caption{Multi-agent AI Framework of Network Threat Detection}
\label{fig: AAI}
\end{figure}

Periodically, Collaborator~1 also computes graph similarity: Measures the weighted Jaccard between the current user graph $G_t$ and a reference graph (e.g., $G_{t-\tau}$ or a baseline).  A drop in this similarity score signals an anomalous shift in user behavior, even if the classifier did not flag it.  These signals are stored for reporting.

The Supervisor agent handles user interaction.  Upon a user query (e.g., 'What recent user behaviors look suspicious?'), the supervisor LLM generates structured queries for the graph database (using standard query languages or prompt-based retrieval). It may ask Collaborator~1 to list subgraphs around suspicious nodes, or ask Collaborator~2 for classification probabilities.  It then synthesizes these into a human-readable interpretation (e.g., 'User X's recent file access pattern is unusual given their history' or 'A new device connection to server Y matches no known normal behavior'). The supervisor thus bridges the gap between automated graph analytics and analyst understanding.

In the following, we embed theoretical analysis into key parts of this system.  First, we analyze the weighted log-likelihood used by Collaborator~2.  Second, we bound the deviation of the weighted Jaccard similarity (used by Collaborator~1).  Third, we give a game-theoretic view of the LLM agents’ collaboration in Section \ref{llm_collaboration_game}.

\begin{figure} [h!]
\centering
\includegraphics[scale=0.3]{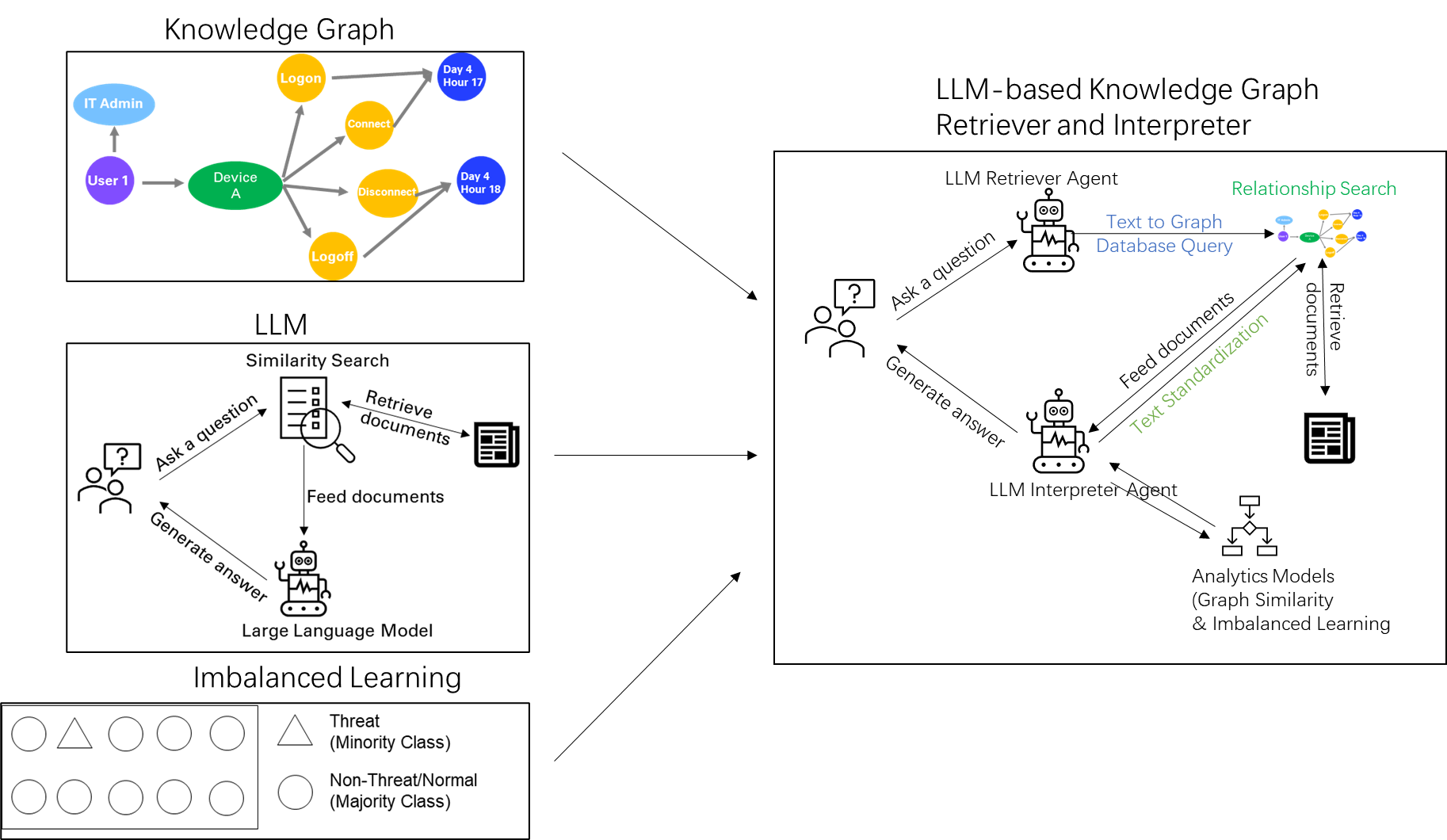}
\caption{LLM-based Knowledge Graph Retriever and Interpreter}
\label{fig: llm_interpreter}
\end{figure}

\subsection{LLM Query Process}
Figure \ref{fig: llm_interpreter} illustrates how agents exchange information via natural language and graph queries. Each agent can translate between text and data operations: graph queries, classifier invocations, and explanation generation. Through iterative prompting, the agents refine their analysis: for example, the supervisor may refine a question like “Why is node X flagged?” and get successive clarifications from collaborators before answering.

Compared to the typical query process of LLM in Figure \ref{fig: LLM QA}, the following additional functionalities are added to our multi-agent LLM query process. 
\begin{itemize}
    \item Multiple types of knowledge base including user-activity knowledge graph and documents: This ensures more comprehensive information to be considered.
    \item Interpretation from analytic models including graph similarity and imbalanced learning: This avoids the fine-tuning of LLM for specific purposes, saving money and improving efficiency.
\end{itemize}

\subsection{Threat Detection Metrics}
We use two complementary detection signals.  The first is the weighted classifier score from Collaborator~2.  The second is a similarity-based anomaly score: the weighted Jaccard index between the current graph and a reference.  Concretely, if $G_t$ and $G_{t-1}$ have total weights $W_{t-1}$ and $W_t$, we compute Equation \ref{jaccard_zhu}. A value $J_t<1$ indicates new or changed edges.  In Section~5.2.4 we show that if the total weight change $\Delta$ between $G_{t-1}$ and $G_t$ is small relative to $W$, then $J_t$ is close to 1; therefore, an abrupt drop in $J$ flags an anomaly.

\begin{equation} \label{jaccard_zhu}
\begin{aligned}
J_t = J(G_{t-1},G_t) 
= \frac{\sum_e \min(w_{t-1}(e),w_{t}(e))}{\sum_e \max(w_{t-1}(e),w_{t}(e))}
\end{aligned}
\end{equation}

\subsection{Weighted Imbalanced Classifier}
Collaborator~2 uses a weighted logistic regression.  Let $(x_i,y_i)$ be data points (features $x_i$ from the graph, label $y_i\in\{0,1\}$ for benign/threat).  We assign weight $w_i=\alpha$ if $y_i=1$ (threat) and $w_i=1$ if $y_i=0$.  The weighted log-likelihood loss can be found in Equation \ref{log_likelihood_data}, where $p_\theta(x)=1/(1+\exp(-\theta^\top x))$.  This upweights rare positives. We analyze the statistical behavior of this estimator next.

\begin{equation} \label{log_likelihood_data}
\begin{aligned}
    \min_{\beta, \lambda} & -\sum_{i=1}^{m}[\lambda_i y_i log(\pi(\beta^T x_i)) + \\ 
    & (1-y_i)log(1-\pi(\beta^T x_i))]
\end{aligned}
\end{equation}

\subsubsection{Imbalanced Learning Consistency}
We formalize the consistency and generalization of the weighted log-likelihood estimator in the imbalanced setting.  Let the data $(x_i,y_i)$ be i.i.d.\ with true model $P(y=1|x)=p_{\theta^*}(x)$.  The population risk is $R(\theta)=E_{X,Y}[w(Y)\ell(Y,p_\theta(X))]$, where $w(1)=\alpha>1$, $w(0)=1$, and $\ell(y,p)=-[y\ln p+(1-y)\ln(1-p)]$.  The empirical risk is $\hat R_n(\theta)=\frac{1}{n}\sum_{i=1}^n w(y_i)\ell(y_i,p_\theta(x_i))$.

\begin{theorem}
Under standard regularity conditions (bounded features, model identifiable), the minimizer $\hat\theta_n$ of the empirical weighted loss is consistent: $\hat\theta_n \stackrel{p}{\to}\theta^*$ as $n\to\infty$.  Moreover, with probability at least $1-\delta$, the uniform deviation
\begin{equation} \label{eq_theorem1}
\begin{aligned}
\sup_{\theta}\big|\hat R_n(\theta) - R(\theta)\big| \;=\; O\Big(\sqrt{\frac{\log(1/\delta)}{n}}\Big)
\end{aligned}
\end{equation}
assuming $x$ has bounded norm.  Consequently, $\hat\theta_n$ has generalization error converging at rate $O(1/\sqrt{n})$.
\end{theorem}

\begin{proof}
\ The consistency and asymptotic normality of weighted likelihood estimators are well known ~\citep{Xue2021}.  Here the weighted loss is a convex surrogate of the classification 0-1 loss and is Lipschitz in $\theta$ for bounded $x$.  By standard uniform convergence (e.g.,\ Hoeffding’s inequality for bounded loss), we have with high probability
\begin{equation} \label{eq_theorem1_2}
\begin{aligned}
\sup_{\|\theta\|\le B} \Big|\hat R_n(\theta)-R(\theta)\Big| \le O\Big(\sqrt{\frac{\log(1/\delta)}{n}}\Big)
\end{aligned}
\end{equation}
Thus the empirical minimizer $\hat\theta_n$ satisfies $R(\hat\theta_n)\approx R(\theta^*)$.  As $n\to\infty$, this implies $\hat\theta_n\to\theta^*$.  Detailed asymptotic normality can be derived from the weighted score equations, yielding $\sqrt{n}(\hat\theta_n-\theta^*)\to N(0,I^{-1})$ under regularity.  This gives the claimed rates.
\end{proof}

This theorem shows that despite class imbalance (via weight $\alpha$), the weighted MLE is consistent and converges at the usual $O(1/\sqrt{n})$ rate.  In practice, we choose $\alpha$ (or class-based weights) to balance precision and recall on rare threats.

\subsubsection{Graph Similarity Bound} \label{graph_similarity_bound}
Next we quantify how much the weighted Jaccard similarity can drop under bounded graph updates.  Let $G$ and $H$ be two graphs with total (edge) weight $W$ each, and suppose the total absolute change in edge weight is $\Delta = \sum_e |w_H(e)-w_G(e)|$.  We assume $W_H=W_G=W$ after possible normalization.  Then:

\begin{theorem}
\label{thm:jaccard}
Let $G$ and $H$ be weighted graphs with total weight $W$ and weight-change $\Delta$.  Then the weighted Jaccard similarity satisfies
\begin{equation} \label{eq_theorem2_0}
\begin{aligned}
J(G,H) \;=\; \frac{\sum_e \min(w_G(e),w_H(e))}{\sum_e \max(w_G(e),w_H(e))} \\ \;\ge\; \frac{W-\Delta}{W+\Delta}\;=\;1 - \frac{2\Delta}{W+\Delta}
\end{aligned}
\end{equation}
In particular, if $\Delta \ll W$, then $J(G,H)\approx 1 - 2\Delta/W$. 
\end{theorem}

\begin{proof}
\ By definition, for each edge $e$ we have
$\min(w_G(e),w_H(e)) \ge w_G(e) - |w_H(e)-w_G(e)|$ and 
$\max(w_G(e),w_H(e)) \le w_G(e) + |w_H(e)-w_G(e)|$.  Summing over edges:
\begin{equation} \label{eq_theorem2_1}
\begin{aligned}
& \sum_e \min(w_G(e),w_H(e)) \\ & \;\ge\; \sum_e w_G(e) - \sum_e|w_H(e)-w_G(e)| \;=\; W - \Delta
\end{aligned}
\end{equation}
\begin{equation} \label{eq_theorem2_2}
\begin{aligned}
&\sum_e \max(w_G(e),w_H(e)) \\ &\;\le\; \sum_e w_G(e) + \sum_e|w_H(e)-w_G(e)| \;=\; W + \Delta
\end{aligned}
\end{equation}
Therefore 
\begin{equation} \label{eq_theorem2_3}
\begin{aligned}
J(G,H) \;=\; \frac{\sum_e \min(\cdot)}{\sum_e \max(\cdot)} \;\ge\; \frac{W-\Delta}{W+\Delta}
\end{aligned}
\end{equation}
Rearranging yields the stated bound $1-J \le 2\Delta/(W+\Delta)$.
\end{proof}

This bound shows that if the total graph weight changes are small, the Jaccard similarity remains near 1.  Only when $\Delta$ is large compared to $W$ (anomalous burst of edges) will $J$ drop significantly.  Thus the Jaccard metric provides a calibrated anomaly score for graph changes.

\subsection{LLM Collaboration as a Cooperative Game} \label{llm_collaboration_game}
We model the LLM agents’ collaboration as a cooperative game.  Let $N=\{1,2,3\}$ be the agents and $v(S)$ the total detection value when coalition $S\subseteq N$ cooperates (e.g.,\ the combined performance).  We assume $v(\emptyset)=0$ and that $v$ is \emph{supermodular} (convex): for all $S,T\subseteq N$,
\begin{equation} \label{LLM_collaboration}
\begin{aligned}
v(S\cup T) + v(S\cap T) \;\ge\; v(S)+v(T)
\end{aligned}
\end{equation}
This means adding an agent yields larger incremental benefit when the coalition is larger.  

\begin{theorem}
If the characteristic function $v$ is supermodular (convex), then the cooperative game has a non-empty core.  Equivalently, there exists a stable allocation of $v(N)$ among the agents such that no subset of agents would gain by deviating.  In particular, the Shapley value lies in the core (every convex game’s core is non-empty) \citep{Shapley1971}.
\end{theorem}

\begin{proof}
\ This is a classical result (Shapley 1971) on convex cooperative games.  Supermodularity implies the Bondareva–Shapley balancedness condition holds, guaranteeing a non-empty core.  One can verify that the Shapley value allocation $\phi_i = \frac{1}{|N|!}\sum_{\pi} [v(\text{pre}_{\pi}(i)\cup\{i\}) - v(\text{pre}_{\pi}(i))]$ satisfies $\sum_i\phi_i=v(N)$ and $\sum_{i\in S}\phi_i\ge v(S)$ for every $S\subseteq N$.  Thus the Shapley value is in the core \citep{Shapley1971}, so the core is non-empty.  The intuition is that when agents’ contributions complement each other (convexity), no group can improve by breaking away.
\end{proof}

In our multi-agent setting, supermodularity is reasonable: having all three agents cooperate (graph builder, classifier, interpreter) yields more detection value than any subset.  The theorem guarantees there exists a stable way to share credit among the agents so that collaboration is beneficial for each subset.

\section{Modeling Process and Results}\label{sec:model}

\subsection{Knowledge Graph Creation} \label{knowledge_graph_creation}

A graph $G$ uses the nodes and edges to represent relational information about users, computing devices, and activities, as shown in Figure \ref{fig: graph_schema}.

\begin{itemize}
\item Nodes $V$: A node represents a user, user role, device, activity type (i.e., logon, email, file access, removable connect, removable disconnect, web visit, logoff) and activity time.
\item Edges $E$: The edges connect the user, the user role, the device, the activity type, and the activity time, which describe what the user did what activity on the device at what time. 
\end{itemize}

\begin{figure}[h!]
\centering
\includegraphics[scale=0.5]{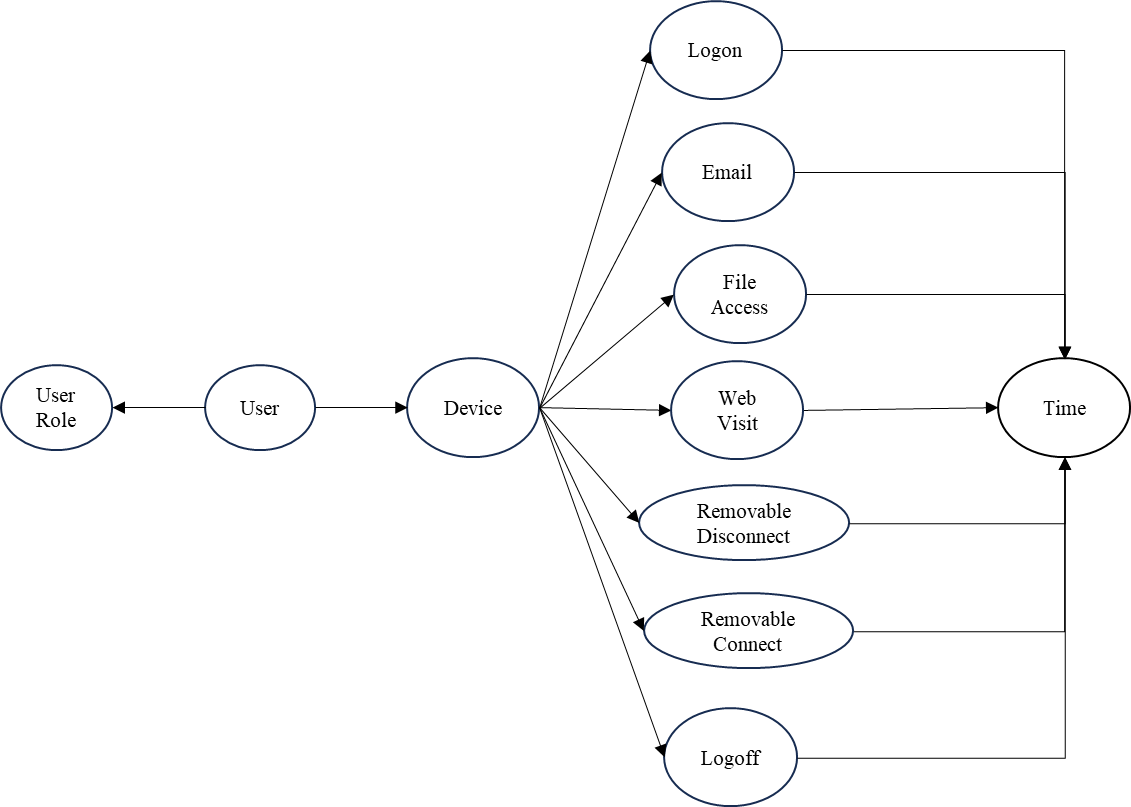}
\caption{User-Activity Knowledge Graph Schema}
\label{fig: graph_schema}
\end{figure}

\subsection{Graph Pruning and Graph Weighting using Imbalanced Learning techniques} \label{graph_prunning_weighting}

To reduce redundant and insignificant information in the Knowledge Graph and improve the algorithm efficiency, we prune and weigh the graph using Imbalanced Learning techniques, which evaluate how the information represented by the nodes and edges in the graph are related to the threats in the historical data. First, the numerical representations of the user activities in the graph and the historical threats are created as independent variables (i.e., features) and dependent variable respectively. Then, their relationships are examined through information value and variable clustering techniques filter out the independent variables with weak predictive power or redundant information. The nodes and edges representing weak or redundant information are pruned from the graph. Lastly, a predictive model is trained with a customized imbalanced learning technique to predict whether a logon session is a threat. The predicted value from this predictive model is used to weigh the nodes in the graph. 

\subsubsection{Feature Creation} \label{feature_creation}

In the predictive model, the dependent variable is a binary variable with 1 indicating a threat logon and 0 indicating a normal logon. The independent variables are 56 variables (that is, features) representing the current and past activities of the users, such as the number of executable files running in the current session and the number of executable files running in the past sessions. 

\subsubsection{Feature Selection} \label{feature_selection}

The relationships between these independent variables and the dependent variable are then examined through the information value, and the interrelationships among these independent variables are examined through the clustering of variables. 16 independent variables are selected to be used in the modelling. The nodes and edges representing the information of these 16 variables are kept.

\subsubsection{Imbalanced Learning} \label{imbalanced_learning}

In historical data, the percentage of threats detected is 0. 34\%. To mitigate data bias, the weight of each training sample $\lambda_i$ is first learned through a custom log-likelihood function ~\citep{zhang2022improving}, as shown in Equation \ref{log_likelihood_data}, from the training data, where $i$ is the training sample index, $y_i$ is the dependent variable value of the training sample $i$, $x_i$ is the independent variable vector value of the training sample $i$, and $\beta$ is the coefficient vector of independent variables. The training data are 70\% of the historical data while the remaining 30\% of historical data are used as validation data to evaluate the model performance later. 

The learned sample weights $\lambda_i$ are applied to the training process of machine learning models (e.g., Gradient Boosting Model). Two models are built in this experiment. Their performance is evaluated on the validation data based on the metrics of Gain and Area under Precision-Recall Curve, as shown in Table \ref{model_performance}. 

\begin{itemize}
\item Model 1: Gradient Boosting Model trained without learnable weights
\item Model 2: Gradient Boosting Model trained with learnable weights from Equation \ref{log_likelihood_data}
\end{itemize}

\begin{table}[h!]
\centering
\begin{tabular}{@{}l@{\quad}c@{\quad}c@{}}
\hline 
Performance Metric    & Model 1 & Model 2  \\ 
\hline 
\% captured true threats  \\among all true threats (gain)  \\ at top 3\% predicted risky logons & 56\% & 60\% \\
\hline 
\% captured true threats \\ among all true threats (gain) \\ at top 30\% predicted risky logons & 95\% & 98\% \\
\hline 
Area under Precision-Recall Curve   & 0.186 & 0.204   \\ \hline
\end{tabular}
\caption{Model Performance}
\label{model_performance}
\end{table}

Compared to Model 1, Model 2 trained with learnable weights can capture 4\% more true threats at the top of 3\% predicted risky logons and 3\% more of true threats at top 30\% predicted risky logons. The overall improvement is around 2\% under different probability cut-offs to convert the predicted threat probability into binary values. This can potentially prevent the loss of \$0.2 billion in 2021, \$0.3 billion in 2022, and \$0.4 billion in 2023 ~\citep{IC3}. 

Model 2 is used to predict the probability of threat for current and new logons. The predicted probabilities are used to weigh the logon activity nodes and their connected user nodes and device nodes in the graph.

\subsubsection{Graph Similarity} \label{graph_similarity}

To measure the change in user activity over time, we first build the current activity graph and the previous activity graph. Then the weighted similarity (e.g., Weighted Jaccard Similarity) between these two graphs is computed. 



To show the result, take the user CSC0217 as an example in Figure \ref{fig: graph_change_score}. Its current activity graph shows it logs on the device PC5866 in the afternoon and connects to a removable device, while its past activity graph shows it logged on the devices PC3742, PC6377, and PC2288 in the morning and visited some websites. The Jaccard Similarity Score between these two graphs is very small, resulting in a high activity change score. 

\begin{figure} [h!]
\centering
\includegraphics[scale=0.3]{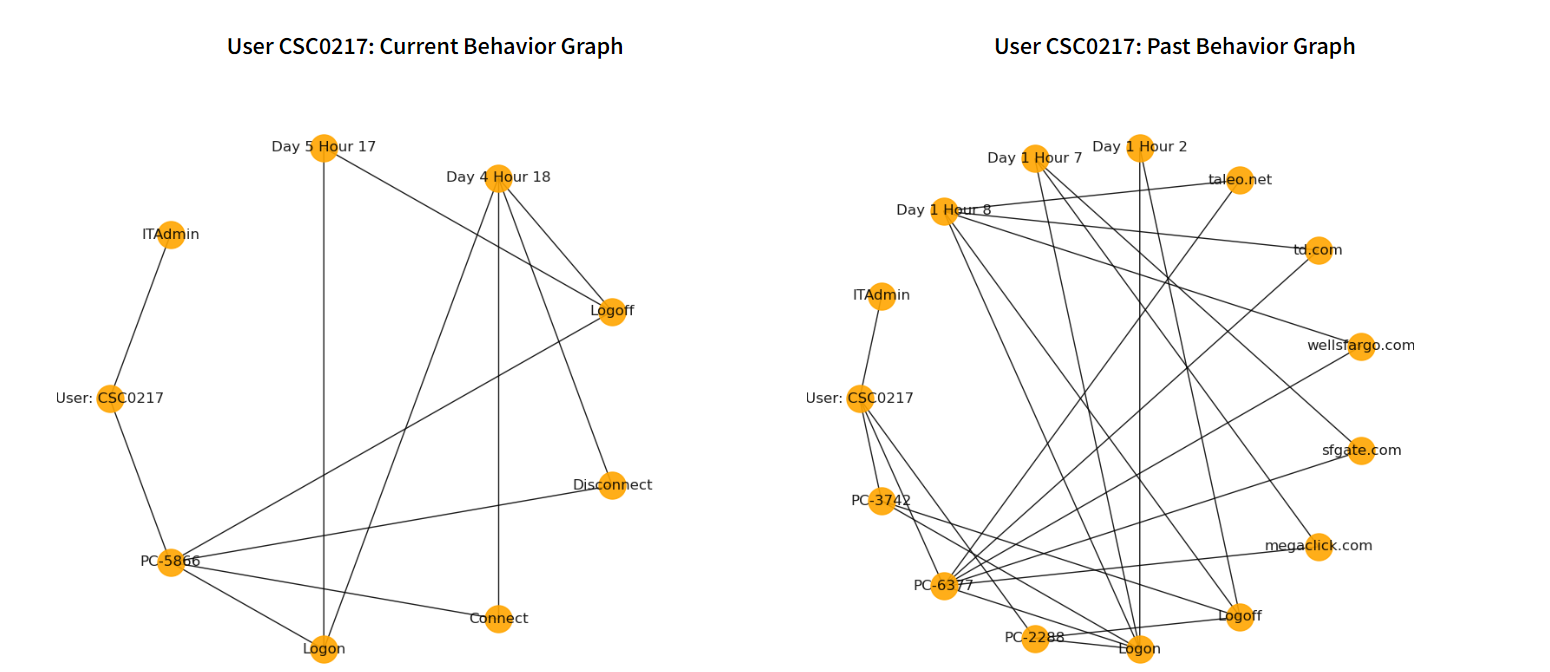}
\caption{User CSC0217 Activity Graph and Change Score}
\label{fig: graph_change_score}
\end{figure}

\subsection{Graph Retrieval and Interpretation using Large Language Model} \label{subsec: LLM}
\subsubsection{Graph Schema Creation - Extended}

The Content of the user’s Email, File, and Web Visits are important behavior factors. Due to its long text size, it typically cannot be integrated efficiently using traditional knowledge graph techniques. Thanks to the advancement of text embedding and vector databases, they can be efficiently integrated into Knowledge Graph by linking to their embedding values stored in a vector database, as shown in Figure \ref{fig: graph_schema_ext}.

\begin{figure} [h!]
\centering
\includegraphics[scale=0.5]{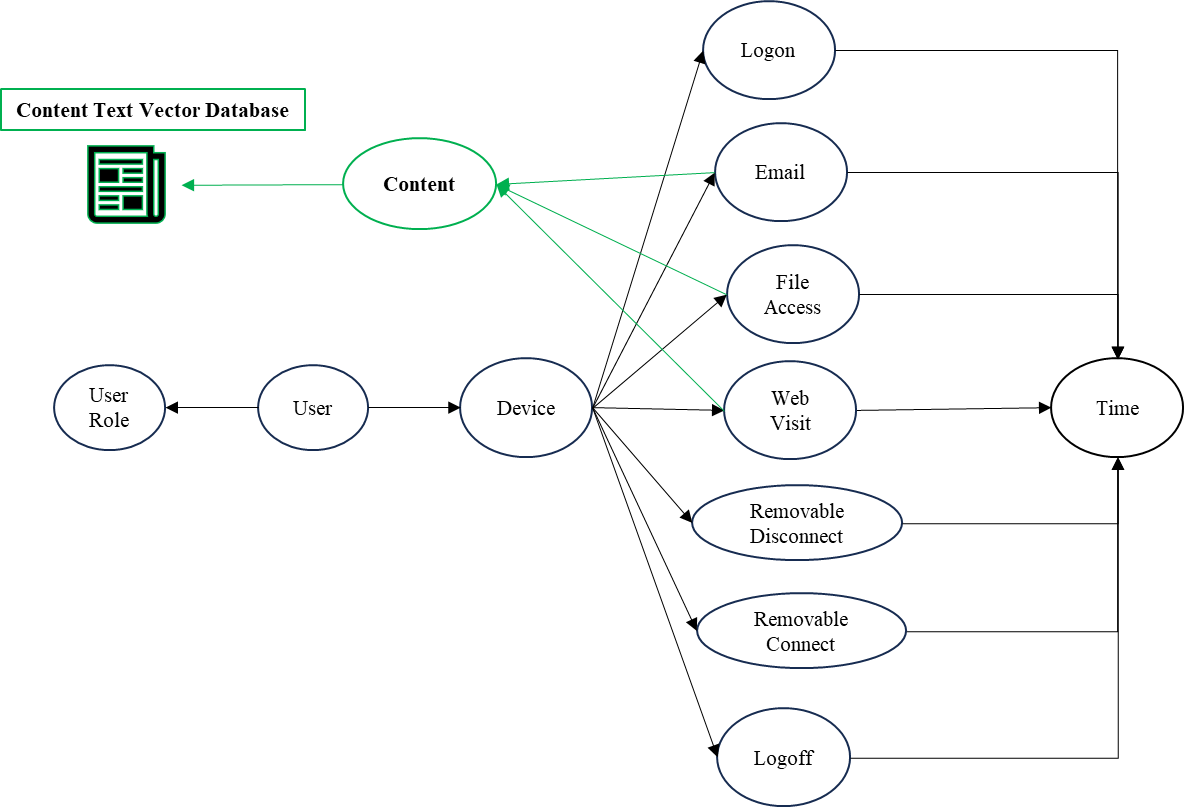}
\caption{User-Activity Knowledge Graph Schema Extended with Content Texts}
\label{fig: graph_schema_ext}
\end{figure}

\subsubsection{Large Language Model as Retriever and Interpreter}

To more efficiently retrieve and interpret users' activities and risks for network threat diagnosis and detection, LLM agents do the following work, as shown in Figure \ref{fig: llm_interpreter}.

\begin{itemize}
\item	LLM translates the user’s question in English into Graph Database query language and does the relationship-based search in Knowledge Graph.
\item	LLM standardizes the text data (e.g., user role) and improves the data quality in Knowledge Graph.
\item	LLM summarizes the user activity information and the content visited from the Knowledge Graph.
\item LLM calls to compute the graph similarity between the current activity graph and the past activity graph for the likelihood of being an unknown threat.
\item LLM calls to learn through the imbalanced learning model for the likelihood of being a known threat.
\item	LLM interprets the user’s activities for the user’s interest and intention based on its own training knowledge base from the whole Web.
\end{itemize}

An application demo of this process can be found in Figure \ref{fig: demo_data_flow}. In the demo, after we ask about a user's activity changes in the current time period compared with a historical time period or reference time period, we can get the answers about the user's activity summaries, changes, and risk interpretations. For example, as shown in Figure \ref{fig: demo}, we ask about the user Lisa's activity change in January 2024 compared to December 2023. In the generated answer, we can get to know Lisa's activity changes by shifting from a mix of web visits and logon/logoffs in December 2023 to solely Logon/Logoffs with Removable Connect and Disconnect on different devices, which explains its likelihood of unknown threat is 80\% and its likelihood of known threat is 70\%. 

\begin{figure} [h!]
\centering
\includegraphics[scale=0.4]{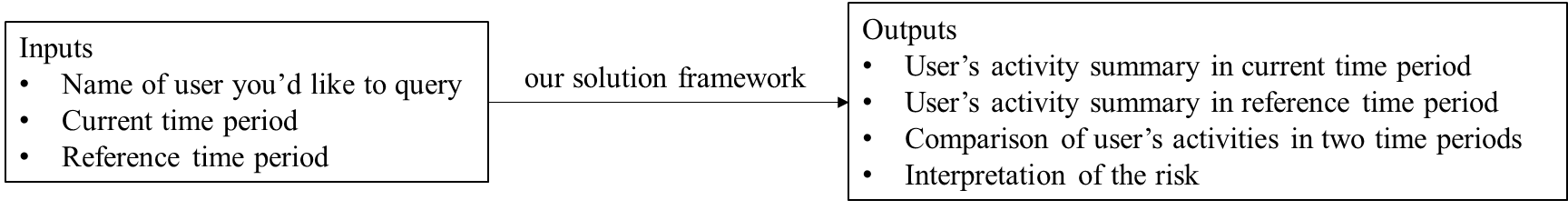}
\caption{Application Demo Data Workflow}
\label{fig: demo_data_flow}
\end{figure}

\begin{figure} [h!]
\centering
\includegraphics[scale=0.7]{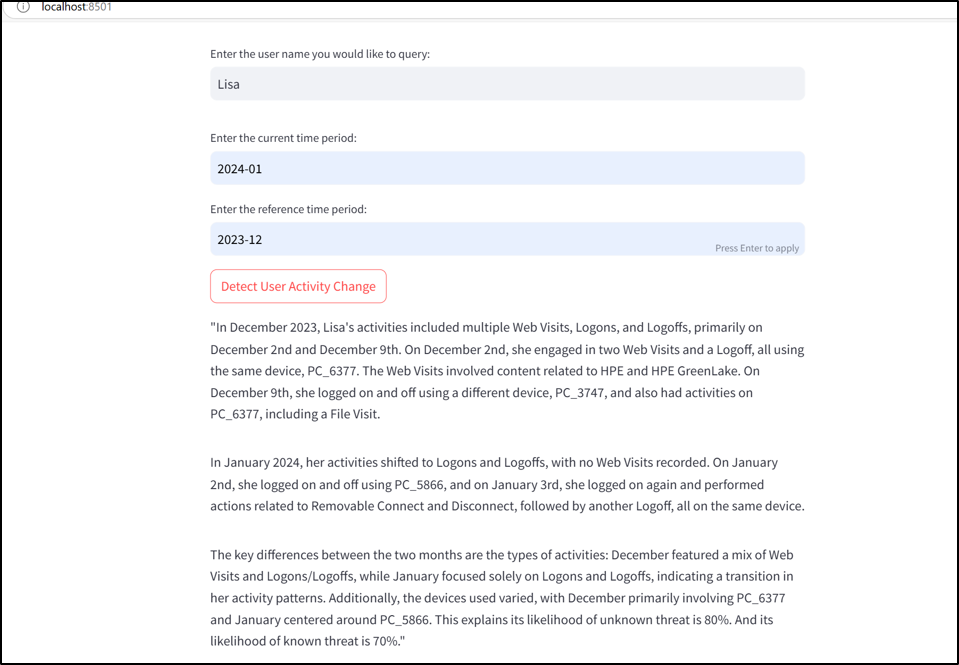}
\caption{Application Demo Results}
\label{fig: demo}
\end{figure}

In this experiment, the open source graph database Nebula is used to build and store user activity graphs, the Python package llama-index is used to index Knowledge Graph, the text-embedding model Text-embedding-3-large is used to vectorize the contents, and the LLM model GPT-4o mini is used to query and generate the answer. 

For the content data used in this experiment, we first tested the contents in the Insider Threat Test Dataset, but got the response: “The text appears to be a mix of random words and phrases that do not form coherent sentences or convey a clear message. It seems to be a jumble of disconnected information that does not provide a clear context or topic. It is difficult to interpret any specific interest or meaning from this text. Alternatively, for demonstration purposes, we used the online public documents to build the content nodes in the graph. 

\section{Future Work}
We outline a formal roadmap for system benchmarking and performance goals:

\begin{itemize}
    \item \textbf{Performance Metrics:} Define specific targets for latency and throughput.  For example, under a streaming rate of thousands of events per second, the system should update the knowledge graph and compute anomaly scores with minimal delay (e.g., $\le100$ ms per batch).  We will measure graph update latency (time to process $\Delta E_t$), classifier throughput (events/sec), and LLM response time.  These goals can be guided by current enterprise requirements and profiling studies.
    \item \textbf{Scalability:} Evaluate scalability under increasing data rates and graph sizes.  We plan stress tests with longer time windows, more users, and faster event streams to ensure that the system remains responsive.  Implementations will leverage incremental graph database technologies and parallel processing for ML inference.  We will also optimize LLM calls (e.g., through caching or summarization) to handle frequent queries.
    \item \textbf{Generalizability:} Test across multiple threat scenarios and datasets.  Beyond the CERT insider dataset, we will apply our framework to other network datasets (e.g., intrusion logs, phishing datasets) to assess robustness.  Cross-validation experiments will check how the weighted classifier and similarity detector perform on different threat types.  In addition, we will evaluate the LLM explanations through a human expert review to ensure interpretability across contexts.
    \item \textbf{Baseline Comparisons:} Establish baselines for detection and performance.  These include a non-graph ML approach (e.g.,\ a random forest on raw features), a standalone graph anomaly detector (e.g.,\ based on AD kernels \citep{Akoglu2015}), and a single LLM interpretation.  We will compare detection accuracy (ROC, precision/recall) and latency against these baselines to quantify the benefit of our combined approach.
\end{itemize}

This roadmap ensures that future work on this system will be evaluated against concrete performance targets and diverse scenarios, providing rigorous and repeatable results.

\section{Conclusion}\label{sec:Conclusion}
To help enhance existing practices of using analytics, machine learning, and artificial intelligence methods to detect network threats, we propose a multi-agent AI framework, which helps effectively quantify the threat risk of users and computing devices based on their complex activities. In this solution, a Knowledge Graph is used to analyze user activity patterns and threat signals. Then, imbalanced learning techniques are used to prune and weigh the knowledge graph, and also calculate the risk of known threats. Finally, a LLM is used to retrieve and interpret information from Knowledge Graph and Imbalanced Learning Model. This approach integrates the strengths of LLMs in contextual reasoning with the structured relationship modeling capabilities of Knowledge Graphs to monitor, predict, and explain potential threats as they unfold. The synergy between these components ensures both depth and immediacy in threat detection, making the system highly effective in dynamic environments. The preliminary results show that the solution improves threat capture rate by 3\%-4\% and adds natural language interpretations of the risk predictions based on the user activities. Furthermore, a demo application has been built to show how the proposed solution can be deployed and used.

\section*{Acknowledgement}
We truly appreciate the support and feedback from our team members and leaders, since we initialized this Hackathon project using open data to solve open problems in 2023. Their names will be disclosed when we reach the publication-ready phase.

\bibliographystyle{abbrvnat}  
\bibliography{ref} 
\end{document}